\documentclass{article}
\usepackage{ijcai22}

\usepackage{times}
\usepackage{soul}
\usepackage{url}
\usepackage[hidelinks]{hyperref}
\usepackage[utf8]{inputenc}
\usepackage[small]{caption}
\usepackage{graphicx}
\usepackage{amsmath}
\usepackage{amsthm}
\usepackage{booktabs}
\usepackage{algorithm}
\usepackage{paralist}

\usepackage{algorithmicx}
\usepackage{algpseudocode}
\usepackage{multirow}

\DeclareMathOperator{\given}{|}

\newtheorem{example}{Example}
\newtheorem{theorem}{Theorem}
\newtheorem{definition}{Definition}
\newtheorem{corollary}{Corollary}

\title{Exchangeability-Aware Sum-Product Networks}

\author{
Stefan Lüdtke$^1$\footnote{Contact Author}\and
Christian Bartelt$^1$\And
Heiner Stuckenschmidt$^{2}$
\affiliations
$^1$Institute for Enterprise Systems, University of Mannheim, Germany\\
$^2$Data and Web Science Group, University of Mannheim, Germany\\
\emails
\{luedtke,bartelt\}@es.uni-mannheim.de,
heiner@informatik.uni-mannheim.de
}

\begin{document}

\maketitle

\begin{abstract}

Sum-Product Networks (SPNs) are expressive probabilistic models that provide exact, tractable inference.
They achieve this efficiency by making use of local independence.

On the other hand, mixtures of exchangeable variable models (MEVMs) are a class of tractable probabilistic models that make use of \emph{exchangeability} of discrete random variables to render inference tractable. Exchangeability, which arises naturally in relational domains, has not been considered for efficient representation and inference in SPNs yet. 

The contribution of this paper is  a novel probabilistic model which we call \emph{Exchangeability-Aware Sum-Product Networks} (XSPNs). It contains both SPNs and MEVMs as special cases, and combines the ability of SPNs to efficiently learn deep probabilistic models with the ability of MEVMs to efficiently handle exchangeable random variables.
We introduce a structure learning algorithm for XSPNs and empirically show that they can be more accurate than conventional SPNs when the data contains repeated, interchangeable parts. 
\end{abstract}

\section{Introduction}

The accurate representation of probability distributions and efficient inference in such distributions is fundamental in machine learning.
Recently, probabilistic models based on deep neural networks, like Variational Autoencoders \cite{kingma2014auto}, neural autoregressive models \cite{larochelle2011neural} and normalizing flows \cite{rezende2015variational} 
 have received a lot of attention. They have been very successful for tasks like data generation, anomaly detection, and prediction.
 However, these models lack efficient and exact inference capabilities. 
 
Probabilistic circuits (PCs) like Probabilistic Sentential Decision Diagrams \cite{liang2017learning} or Cutset Networks \cite{rahman2014cutset} are deep probabilistic models that permit efficient, exact marginal inference. 
Sum-Product Networks (SPNs) \cite{poon2011sum}, one of the most prominent PCs, represent probability distributions as a  computation graph that consists of sum nodes (representing mixtures) and product nodes (representing independent factors). Time complexity of marginal inference in SPNs is linear in the network size.

In this paper, we are specifically interested in efficient inference in (and learning of) discrete distributions involving repeated, interchangeable components. Such distributions arise naturally in \emph{relational} domains, consisting of multiple, interrelated entities, as illustrated by the following example: 
\begin{example}
\label{expl:intro-new}
 Suppose $n$ people are invited to a workshop, and we want to estimate how many of them will attend. We assume that attendance of each person depends on whether the topic of the workshop is considered ``hot''. Additionally, attendance of each person depends on how many of the other people will attend. 
\end{example}
In this domain, attendance of each person is not independent, but depends on the attendance of all other people. 
Instead, however, the probability that exactly $k$ people attend does not depend on the specific \emph{identities} of those $k$ people. 
More formally, the random variables representing attendance are \emph{exchangeable}---their joint distribution is invariant under permutation of their assignment. 
Unfortunately, tree-structured SPNs that are generated by SPN structure learning algorithms (as well as graphical models based on the notion of conditional independence) do not efficiently encode distributions over exchangeable RVs. 

Still, similar to conditional independence, finite exchangeability can substantially reduce the complexity of the model and allow for tractable inference \cite{niepert2014tractability}.
This property is exploited by Mixtures of Exchangeable Variable Models (MEVMs) \cite{niepert2014exchangeable}. MEVMs are tractable probabilistic models that use efficient representations of exchangeable distributions as building blocks. 
MEVMs can be seen as shallow SPNs, consisting of a single sum and product layer, and efficient representations of finite exchangeable sequences at the leaves. 
However, deep SPN architectures can represent some functions more efficiently than shallow architectures \cite{delalleau2011shallow}.
Despite their obvious connection, no attempts have been made yet to combine SPNs and MEVMs into a unified formalism.

The main contribution of this paper is to provide such a unified model, that contains both MEVMs and (conventional) SPNs as special cases. The proposed model combines the ability of SPNs to efficiently represent and learn deep probabilistic models with the ability of MEVMs to efficiently handle exchangeable random variables.
We show that marginal inference in the proposed model---which we call \emph{Exchangeability-Aware} SPNs (XSPNs)---is tractable. The general concept is shown in Figure \ref{fig:concept}.
Furthermore, we introduce a structure learning algorithm for XSPNs, which recursively performs statistical tests of finite exchangeability, and fits an exchangeable leaf distribution when appropriate. 
Finally, we empirically demonstrate that our structure learning algorithm can achieve significantly higher log-likelihoods than conventional SPN structure learning.
This work is a first step towards learning probabilistic circuits for relational domains.

\begin{figure}[tb]
\centering
\includegraphics[scale=0.4]{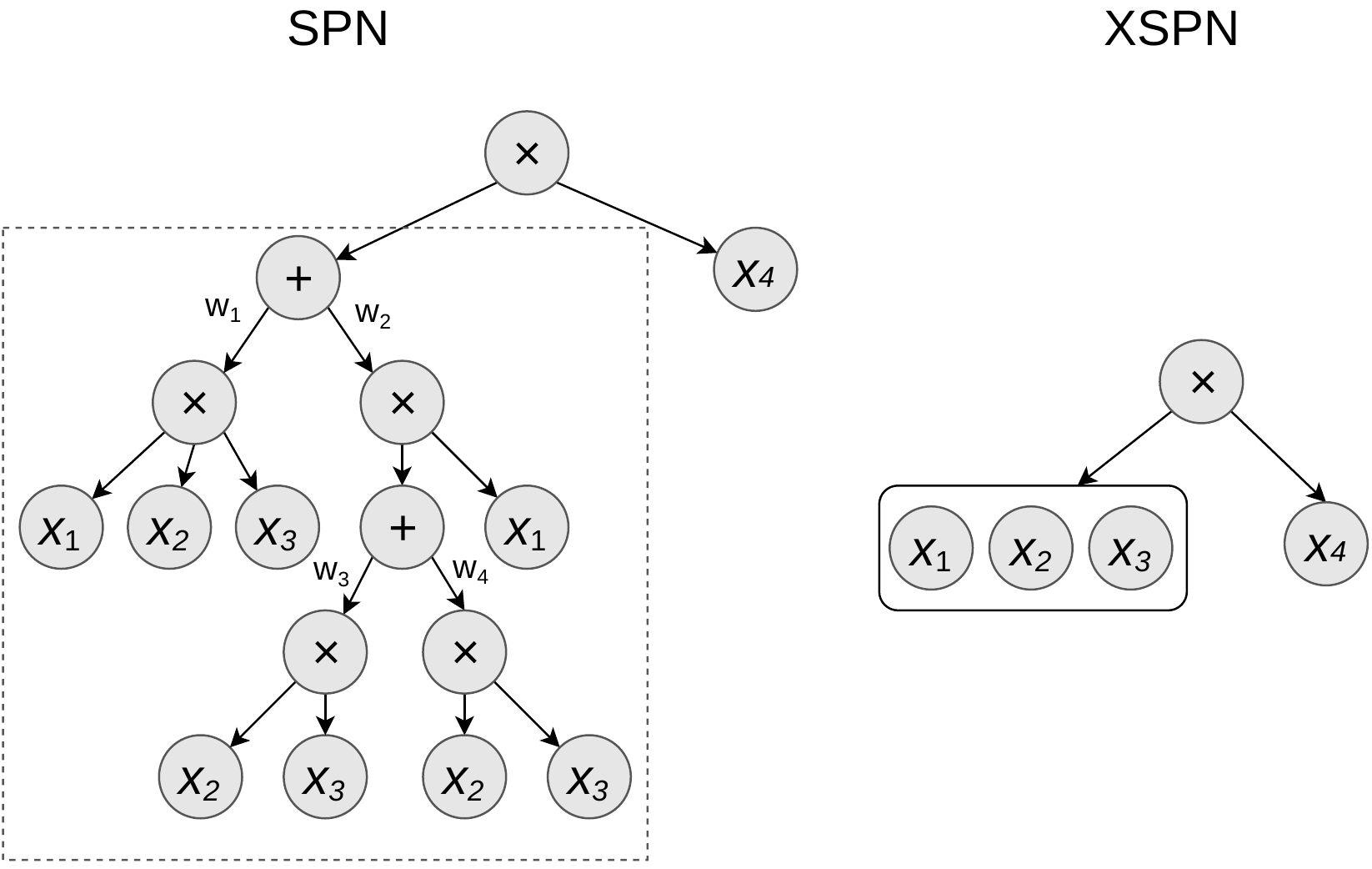}
\caption{Left: SPN generated by LearnSPN from 1000 samples from a distribution of the form $p(X_1,X_2,X_3,X_4) = p(X_1,X_2,X_3) \, p(X_4)$, where $p(X_1,X_2,X_3)$ is exchangeable. This factor contains no (local) independence, thus LearnSPN learns a complicated structure that in essence represents the factor by full enumeration (framed subtree). Right: XSPN representing the same distribution. The XSPN is more compact, because it can directly represent exchangeable distributions efficiently.}
\label{fig:concept}
\end{figure}

\section{Sum-Product Networks}
\label{sec:spns}

\paragraph*{Representation}
An SPN \cite{poon2011sum} is a rooted directed acyclic graph that represents a probability distribution over a sequence of RVs $\mathbf{X} = X_1,\dots,X_n$. 
Each node $N$ in the graph represents a probability distribution $P_N$ over a subset $\mathbf{X}_{\phi(N)} \subseteq \mathbf{X}$ of the RVs, where $\phi(N) \subseteq \{1,\dots,n\}$ is called the \emph{scope} of the node $N$.
An SPN consists of three types of nodes: distribution nodes, sum nodes and product nodes. All leaf nodes of the graph are distribution nodes, and all internal nodes are either sum or product nodes.
In the following, we denote the set of children of the node $N$ as $\text{ch}(N)$.

A distribution node $N$ encodes a tractable probability distribution  $P_N(\mathbf{X}_{\phi(N)})$ over the RVs in its scope. 
Early works on SPNs used univariate distributions, e.g.\ Bernoulli distributions or univariate Normals, but multivariate distributions are possible as well \cite{vergari2015simplifying}.
A product node $N$ represents the distribution $P_N(\mathbf{X}_{\phi(N)}) = \prod_{C \in \text{ch}(N)} P_C(\mathbf{X}_{\phi(C)})$. We require product nodes to be \emph{decomposable}, meaning that the scopes of all children of a product node are pairwise disjoint. 
A sum node $N$ represents the distribution $P_N(\mathbf{X}_{\phi(N)}) = \sum_{C \in \text{ch}(C)} w_C\, P_C(\mathbf{X}_{\phi(C)})$. We require sum nodes to be \emph{complete}, which means that the scopes of all children of the sum node are identical. 

Intuitively, product nodes represent distributions that decompose into independent factors (where decomposability ensures this independence), and sum nodes represent mixture distributions (where completeness ensures that sum nodes represent proper mixtures).
More specifically, product nodes represent \emph{local independence}, i.e., independence that holds conditional on latent variables.
By definition, the distribution represented by an SPN is the distribution defined by its root node.
Figure \ref{fig:concept} (left) shows an example of an SPN. 
Many extensions and generalizations of this general model have been proposed, e.g., models for continuous and hybrid domains \cite{molina2018mixed} or time series models \cite{yu2021whittle}. 

\paragraph*{Inference}
The appealing property of SPNs is that they permit efficient marginal inference, i.e., computing $P(\mathbf{X}' {=} \mathbf{x}')$ for a subset $\mathbf{X}' \subset \mathbf{X}$ of the RVs. This is possible because summation over the marginalized RVs can be ``pushed down'' into the leaf nodes of the SPN \cite{peharz2015theoretical}. Thus, marginal inference reduces to marginalization of the leaves and evaluating the internal nodes of the SPN once. Thus, when marginal inference of the leaf distributions is possible in constant time, marginal inference is linear in the number of nodes of the SPN. This task becomes specifically simple when the leaf distributions are univariate, because the value of marginalized leaves can simply be set to 1.

\paragraph*{Learning}

Efficient inference in SPNs is a result of decomposability and completeness of the SPN. 
Thus, a central challenge is to learn SPN structures that satisfy these constraints.

We focus on LearnSPN  \cite{gens2013learning}, a greedy, top-down structure learning algorithms that creates a decomposable and complete tree-structured SPN. 
The algorithm constructs the tree recursively in a top-down way. At each step, the algorithm performs an independence test on the dataset, and creates a product node if the data can be split into sets of independent RVs. If no independence is identified, a clustering algorithm is used to partition the data into subsets and a corresponding sum node is created, with weights corresponding to the proportion of data in each cluster. 
In both cases, the algorithm is recursively called for the corresponding data subsets. When only a single RV is left in a call, a leaf node is created by estimating a univariate, smoothed distribution of that RV. 
To prevent overfitting, a leaf is also created when the number of instances falls below a pre-defined threshold $m$. In this case, either a full factorization is assumed, or a leaf representing a tractable multivariate distribution is created, e.g. a Chow-Liu tree \cite{vergari2015simplifying}.

Several other SPN learning algorithms have been proposed, e.g.\ a Bayesian structure learning algrithm \cite{trapp2019bayesian}, and an algorithm that generates a random (decomposable and complete) SPN structure and then optimize the parameters of that SPN by EM \cite{peharz2020random}.

\section{Exchangeability-Aware Sum-Product Networks}

We propose Exchangeability-Aware Sum-Product Networks (XSPNs) as novel tractable probabilistic models. They are based on explicitly modeling distributions over partially exchangeable RVs as leaf distributions of an SPN.

\subsection{Finite Exchangeability}
\label{subsec:exchangeability}

Exchangeability of \emph{infinite} sequences of RVs is fundamental in Bayesian statistics. 
Here, we instead focus on exchangeability of \emph{finite} sequences of RVs, as stated by the following definition.

\begin{definition}[Finite Exchangeability]
\label{def:exchangeability}
Let $X_1, \dots, X_n$ be a sequence of RVs with joint distribution $P$ and let $S(n)$ be the group of permutations acting on $\{1,\dots,n\}$. We call $X_1,\dots,X_n$ \emph{exchangeable} iff $P(X_1{=}x_1,\dots,X_n{=}x_n) = P(X_1{=}x_{\pi(1)},\dots,X_n{=}x_{\pi(n)})$ for all $\pi \in S(n)$. 
\end{definition}

Exchangeable RVs are not necessarily independent, and thus a graphical model representation of a distribution over exchangeable RVs can have high tree-width. 
Similarly, conventional SPNs do not encode distributions over exchangeable RVs efficiently (see Figure \ref{fig:concept}).
However, as shown below, exchangeability still allows for an efficient representation and inference.

Full exchangeability is a strong assumption that is not necessary for our purposes. Instead, in this paper, we focus on the weaker property of \emph{partial} exchangeability (first introduced by \cite{de1938condition}), which is based on the notion of sufficient statistics. 

\begin{definition}[Partial Exchangeability]
Let $X_1, \dots, X_n$ be a sequence of RVs with joint distribution $P$, let $\text{dom}(X)$ be the domain of $X$ and let $\mathcal{T}$ be a finite set. The sequence $X_1,\dots,X_n$ is \emph{partially exchangeable} w.r.t.\ the statistic $T: \text{dom}(X_1) \times \dots \times \text{dom}(X_n) \rightarrow \mathcal{T}$ if
\begin{equation*}
T(\mathbf{x}) = T(\mathbf{x}') \text{ implies } P(\mathbf{x}) = P(\mathbf{x}')
\end{equation*}
for all assignments $\mathbf{x}$ and $\mathbf{x}'$. 
\end{definition}

The statistic $T$ partitions the assignments $\mathbf{x}$ into equivalence classes $S_t = \{\mathbf{x} \,|\, T(\mathbf{x}) = t\}$ with identical value $t$ of the statistic $T$ and identical probability.
Partial exchangeability allows to represent a distribution by $|\mathcal{T}|$ parameters, where parameter $w_t$ is the probability of each assignment $\mathbf{x}$ with $T(\mathbf{x})=t$. Using this parametrization, the probability of an assignment $\mathbf{x}$ is given by $P(\mathbf{x}) = \sum_{t \in \mathcal{T}} [T(\mathbf{x}) = t]\, w_t$, where $[\cdot ]$ is the indicator function \cite{diaconis1980generalizations}. 

In the same vein, partial exchangeability allows for efficient marginal and MAP inference. 
Let $\mathbf{e}$ be a partial assignment, and let $\mathbf{x} \sim \mathbf{e}$ denote that $\mathbf{x}$ and $\mathbf{e}$ agree on the values of RVs in their intersection. Furthermore, let
$S_{\mathbf{e},t} = \{\mathbf{x} \,|\, T(\mathbf{x}) = t \text{ and } \mathbf{x}\sim \mathbf{e}\}$ be the set of all assignments that are consistent with evidence $\mathbf{e}$  and that correspond to value $t$ of the statistics.

\begin{theorem}[\cite{niepert2014tractability}]
\label{thm:efficient-inference}
Let $X_1,\dots,X_n$ be partially exchangeable w.r.t. $T$. If we can efficiently,
\begin{compactitem}
\item for all $\mathbf{x}$, evaluate $P(\mathbf{x})$, and
\item for all $\mathbf{e}$ and $t$, decide whether there exists an $\mathbf{x} \in S_{\mathbf{e},t}$ and if so, construct it,
\end{compactitem}
then the complexity of MAP inference is polynomial in $|\mathcal{T}|$. 
If we can additionally compute $|S_{\mathbf{e},t}|$ efficiently, then the complexity of marginal inference is also polynomial in $|\mathcal{T}|$.
\end{theorem}

When the conditions of the theorem are satisfied for a statistic $T$, we call $T$ a tractable statistic. 
As a simple example, consider binary RVs $X_1,\dots,X_n$ and the statistic $T^\#(\mathbf{x}) = T^\#(x_1,\dots,x_n) = \sum_{i=1}^n x_i$, which counts the number of ones in an assignment.
For this statistic, the conditions of Theorem \ref{thm:efficient-inference} are satisfied: For a given $\mathbf{e}$ and $t$, 
constructing an arbitrary $\mathbf{x} \in S_{\mathbf{e},t}$ is straightforward and $|S_{\mathbf{e},t}| = \binom{n-n_{\mathbf{e}}}{t-T(\mathbf{e})}$, where $n_\mathbf{e}$ is the number of values in $\mathbf{e}$.

\paragraph*{Exchangeable Variable Models}
 (EVMs) \cite{niepert2014exchangeable} are tractable probabilistic models based on the notion of partial exchangeability. 
As basic building blocks, they use distributions over partially exchangeable RVs w.r.t. a statistic $T$, which are represented by parameters $w_1,\dots,w_{|\mathcal{T}|}$. 

An Exchangeable Variable Model (EVM) is a product of such exchangeable blocks. Specifically, let $\mathcal{X}$ be a partitioning of the RVs $X_1,\dots,X_n$ into disjoint subsets. An EVM defines a joint distribution via $P(\mathbf{X}) = \prod_{\mathbf{X}_i \in \mathcal{X}} P(\mathbf{X}_i)$, where the factors $P(\mathbf{X}_i)$ are exchangeable blocks, as defined above.
Finally, an MEVM is a mixture of such EVMs. 
Parameter estimation in MEVMs (i.e.\ estimation of the mixture weights, the independence structure inside each EVM and the parameters of the exchangeable blocks) can be done via Expectation Maximization (EM).

Statistical relational models, like Markov Logic Networks \cite{richardson2006markov} or relational Sum-Product Networks \cite{nath2015learning} are another class of probabilistic models that make use of exchangeability for efficient representation and inference. In these models, a high-level, relational structure which implicitly defines which RVs are exchangeable is defined a priori based on domain knowledge. In contrast, MEVMs (as well as our own contribution introduced below) identify exchangeability of RVs purely from the training data. 
Tractable inference in statistical relational models can only be guaranteed for certain sub-classes of models, and can also become intractable when exchangeability breaks due to asymmetrical evidence \cite{jaeger2012liftability}. 

\subsection{Exchangeable Leaf Distributions}

In this section, we introduce a variant of SPNs which we call  \emph{Exchangeability-Aware Sum-Product Network} (XSPN).
An XSPN is an SPN with multivariate leaf distributions $p(X_1,\dots,X_n)$, where $X_1,\dots,X_n$ are partially exchangeable w.r.t. a given statistic $T$. 
Note that the leaves can have different cardinality and can be exchangeable w.r.t. different statistics $T$. Specifically, when all leaves are univariate, XSPNs are equivalent to conventional SPNs.
Due to the fact that marginal inference in distributions of partially exchangeable RVs is tractable, inference in XSPNs is also tractable.

\begin{corollary}
Let $S$ be  the set of all XSPNs with at most $N$ nodes, where the leaf nodes represent distributions over RVs which are partially exchangeable w.r.t. a statistic $T$, which has $|\mathcal{T}|$ possible values. When the conditions from Theorem \ref{thm:efficient-inference} are satisfied for $T$, then time complexity of marginal inference in the XSPNs contained in $S$ is polynomial in $N \cdot |\mathcal{T}|$.
\end{corollary}
\begin{proof}
Marginal inference in XSPNs requires a constant number of evaluations of each inner node, and answering a marginal query at each leaf node \cite{peharz2015theoretical}. There are at most $N$ leaf nodes, and for each leaf node, inference is polynomial in $|\mathcal{T}|$ (Theorem \ref{thm:efficient-inference}).
\end{proof}

For example, time complexity of marginal inference in XSPNs where the leaf distributions use the statistic $T^\#$ with at most $n$ RVs is in $\mathcal{O}(n\,N)$. 
XSPNs unify conventional SPNs (with univariate leaf distributions) and MEVMs in a common framework: Conventional SPNs are a special case of XSPNs where all leaf distributions are univariate, and MEVMs are shallow XSPN with a single sum layer and a single product layer.

\subsection{Learning XSPNs}

\begin{algorithm}[t]
\caption{LearnXSPN($D$,$V$,$m$)}
\label{alg:learnxspn}
\begin{algorithmic}
\State \textbf{Input:} Set of instances $D$ over variables $V$, minimum number of instances $m$
\State \textbf{Output:} SPN representing a distribution over $V$
\If{$|D|<m$} 
\State \Return Exchangeable distribution estimated from $D$ 
\ElsIf{$|V|=1$} 
\State \Return Univariate distribution estimated from $D$
\ElsIf{$V$ are  exchangeable w.r.t. a given statistic $T$}
\State \Return Exchangeable distribution estimated from $D$
\ElsIf{$V$ can be partitioned into independent sets $V_j$}
\State \Return $\prod_j$ LearnXSPN($D$,$V_j$,$m$)
\Else 
\State Partition $D$ into subsets $D_i$
\State \Return $\sum_i$ $\frac{|D_i|}{|D|}$ LearnXSPN($D_i$,$V$,$m$)
\EndIf
\end{algorithmic}
\end{algorithm}

Next, we consider structure learning of XSPNs. The learning algorithm is based on LearnSPN (see Section \ref{sec:spns}) and is shown in Algorithm \ref{alg:learnxspn}. In addition to the usual LearnSPN scheme, the algorithm tests for the presence of partial exchangeability at each recursive call. When partial exchangeability is not rejected by the test, a leaf representing a distribution over partially exchangeable random variables is created directly, i.e.\ the algorithm does not recurse in that case.

Note that this approach for creating leaves is different from other approaches that allow for multivariate leaves, e.g. Chow-Liu trees \cite{vergari2015simplifying}: These other approaches introduce a multivariate leaf as a fallback case, when neither a product node nor a sum node can be created, whereas our algorithm explicitly tests whether the assumptions made in the multivariate leaf distribution are satisfied. 
Still, multivariate leaves can also be created as a fallback case in our algorithm. Here, it is natural to also use distributions over exchangeable RVs.

We want to emphasize that---in contrast to statistical relational models like Markov Logic Networks or Relational Sum-Product Networks \cite{nath2015learning}---XSPNs do not use any prior knowledge about exchangeability structure in the distribution, but detect exchangeability solely based on the data itself.

In the following, we discuss the additional operations that are required by the algorithm in more detail: First, we show how the parameters $w_1,\dots,w_{|\mathcal{T}|}$ of a partially exchangeable distributions can be estimated, and afterwards discuss statistical tests for partial exchangeability. 

\paragraph*{Parameter Estimation}
Let $X_1,\dots,X_n$ be a sequence of RVs that is partially exchangeable w.r.t. statistic $T$, and let  $\{\mathbf{x}^{(i)}\}_{i=1}^N$ be a set of samples. Recall that a distribution over such RVs can be represented by parameters $w_1,\dots,w_{|\mathcal{T}|}$, one parameter for each $t \in \mathcal{T}$. The Maximum Likelihood estimate of parameter $w_t$ is given by \cite{niepert2014exchangeable}\footnote{Note that we directly multiply by $|S_t|^{-1}$ when estimating the parameters, whereas \cite{niepert2014exchangeable} multiply by this factor when evaluating $P(\mathbf{x})$.}
\begin{equation*}
w_t = 1/N\, \sum_{i=1}^N [T(\mathbf{x}^{(i)}) = t]\, |S_t|^{-1}
\end{equation*}
Intuitively, the estimate makes use of the fact that each partially exchangeable distribution can be seen as a mixture of uniform distributions. Each equivalence class $S_t$ corresponds to a mixture component with weight $ 1/N\, \sum_{i=1}^n [T(\mathbf{x}^{(i)}) = t]$. The factor  $|S_t|^{-1}$ ensures that $w_t$ represents the probability of each of the assignments in $S_t$. 
Note that the complexity of computing $|S_t|$ is polynomial for tractable statistics $T$. For the counting statistic $T^{\#}$, this factor can even be computed in constant time, as $|S_t| = \binom{n}{t}$ in that case.

\paragraph*{Exchangeability Tests}
In the following, we discuss the problem of identifying whether the discrete RVs $X_1,\dots,X_n$ are partial exchangeable, given a set of samples $\{\mathbf{x}^{(i)}\}_{i=1}^N$. 
A simple option is given by \cite{niepert2014exchangeable}, who provide a number of necessary conditions for finite exchangeability. For example, when the RVs $X_1,\dots,X_n$ are exchangeable, then $\forall i,j: E[X_i] = E[X_j]$. 

We propose to use a more general approach: Given a set of samples $\{\mathbf{x}^{(i)}\}_{i=1}^N$, we use Pearson's $\chi^2$ test for goodness of fit to test the null hypothesis that the joint distribution $P(\mathbf{X})$ is partially exchangeable w.r.t. a given statistic $T$. 
Specifically, we perform the following steps: (i) estimate the parameters $w_1,\dots,w_{|\mathcal{T}|}$ w.r.t. the statistic $T$ from the samples $\{\mathbf{x}^{(i)}\}_{i=1}^N$, using the ML estimate above; (ii) for each assignment $\mathbf{x}$, compute its expected frequency (assuming that it is distributed according to $P(\mathbf{x}) = \sum_t [T(\mathbf{x}) = t]\, w_t$) and its empirical frequency, (iii) compute the test statistic of Pearson's $\chi^2$ test, and either reject or not reject the null hypothesis.

Unfortunately, the time complexity of this test grows exponentially with the number of RVs $n$, because the expected and empirical frequencies need to be computed for all assignments $\mathbf{x}$. 
Therefore, we propose to approximate the test that works by performing pairwise comparisons.
Specifically, for a sequence $X_1,\dots,X_n$, we test all pairs $(X_i,X_j)$ with $i,j \in \{1,\dots,n\}$, $i < j$ for partial exchangeability, using the test outlined above. When none of these tests rejects the null hypothesis, we assume that $X_1,\dots,X_n$ are partially exchangeable. 

Pairwise exchangeability is a necessary condition of (full) exchangeability \cite[Proposition~1.12]{schervish2012theory}, thus all cases of actual full exchangeability are also identified by the pairwise approach. We rarely encounter false positives in practice, as illustrated empirically in the supplementary material.

The pairwise test needs to perform $\mathcal{O}(n^2)$ $\chi^2$ tests, and for each test, compute a Maximum Likelihood parameter estimate. For the statistic $T^{\#}$, parameter estimation can be done in $\mathcal{O}(N)$ time (where $N$ is the number of samples), and thus the time complexity of all pairwise tests is in  $\mathcal{O}(N\, n^2)$. In practice, this test only has a small effect on overall runtime of LearnXSPN, which is dominated by clustering.

\section{Experimental Evaluation}

We evaluated XSPNs on three types of datasets: Seven synthetic datasets that were constructed to contain (contextually) exchangeable RVs, four real-world datasets consisting of multiple interchangeable parts (where partial exchangeability arises naturally), and 20 commonly used benchmark datasets.
We investigated probability estimation as well as classification performance of XSPNs, compared to conventional SPNs and additional baseline models.

\subsection{Probability Estimation}
\label{subsec:probability-estimation}

\paragraph*{Data}
Four of the synthetic datasets were created by following \cite{niepert2014exchangeable}: We sampled uniformly from $\{0,1\}^{100}$, and then only kept samples that satisfy certain constraints on the number of ones in the sample.
Two datasets (MEVM-s and MEVM-l) were sampled from MEVM models. The conference dataset was sampled from a Markov Logic Network that describes decision-making of people attending or not attending a conference, as introduced in Example \ref{expl:intro-new}.

Additionally, we evaluated our approach on four real-world density estimation tasks. 
The \emph{exams} dataset consists of exam participation information of 487 business students that started their studies in fall term 2019 at the University of Mannheim. 
58 courses were attended by at least 10 of these students between fall 2019 and fall 2021.
Each binary RV represents participation in one of these courses, and each example represents a student.
The \emph{senate} dataset contains all 720 roll call votes in the Senate of the 116th United States Congress, taken from  \cite{lewis2021voteview}.
We only kept votes of the 98 senators that participated in the majority of the votes. Each binary RV represents the votes of a senator, and each example represents a ballot. 
A similar procedure was applied to the votes from the  House of Representatives of the 116th US Congress (\emph{house} dataset) and the 17th German federal parliament (\emph{bundestag} dataset, data taken from \cite{bergmann2018btvote}).

Finally, for completeness, we also evaluated SPNs and XSPNs on 20 benchmark datasets that are commonly used for comparing SPN learning algorithms \cite{gens2013learning}. 
For a more detailed description of the datasets, we refer to the supplementary material.

\paragraph*{Experiments}

We compared the following models: SPNs with Chow-Liu trees as leaf distributions trained via LearnSPN \cite{vergari2015simplifying} (which are closest to XSPNs as they also use multivariate leaf distributions); our XSPNs trained via LearnXSPN; MEVMs (as shallow variants of XSPNs) trained via EM \cite{niepert2014exchangeable}; and \emph{Masked Autoencoders for Distribution Estimation} (MADEs) \cite{germain2015made}. MADEs do not provide tractable marginal inference and are thus not directly comparable to (X)SPNs, but are used as strong baseline density estimators here.
We do not consider methods for modeling distributions over sets \cite{yang2020energy}, because they assume full exchangeability and are thus not suited for the cases investigated here. 
Our implementation\footnote{Available at \url{https://github.com/stefanluedtke/XSPNFlow}} of XSPNs is based on the SPFlow library \cite{Molina2019SPFlow} for Python.

In all SPN models, we used the g-test as independence test and clustered instances via EM for Gaussian Mixture Models.
As proposed by \cite{vergari2015simplifying}, we limited the number of child nodes of each sum and product node to 2. For XSPN leaves, we used the statistic $T^{\#}$ introduced in Section \ref{subsec:exchangeability}.
We performed an exhaustive grid search of the SPN hyperparameters.  
Specifically, we varied the g-test threshold values $\rho\in\{5,15\}$, the minimum number of instances $m \in \{20,200\}$, and the significance level of the $\chi^2$-test for exchangeability $p \in \{0.05,0.1,0.2,0.4\}$. For all models, we used Laplace smoothing with $\alpha = 0.1$. Chow-Liu tree leaf learning failed in some cases using SPFlow, in which case we used fully factorized leaf distributions. 
Parameter settings for the MEVM and MADE models can be found in the supplementary material.

\paragraph*{Results}

\begin{figure}[tb]
\centering
\includegraphics[scale=0.4]{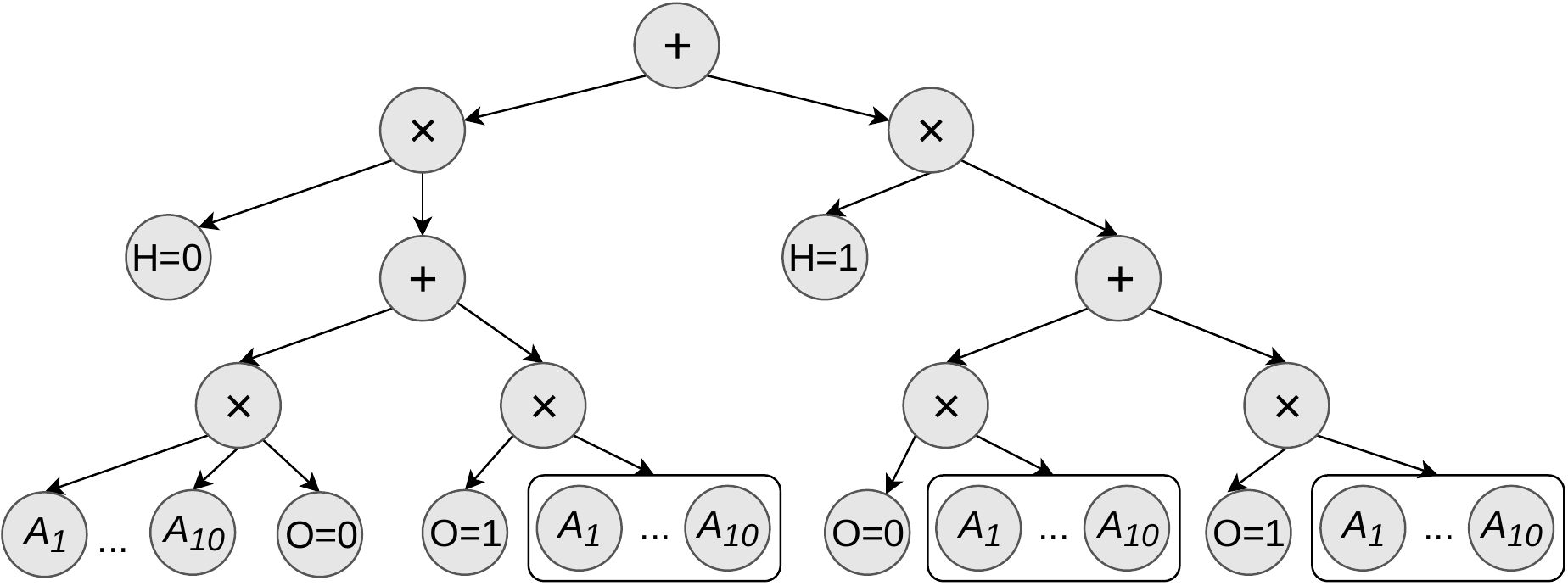}
\caption{XSPN learned for the conference dataset. The XSPN faithfully represents the true distribution underlying the MLN specification: For fixed HotTopic and Overflow, all of the Attends random variables are exchangeable---except for the case where  HotTopic = Overflow = 0, in which case all of the Attends variables are  independent.}
\label{fig:conference-xspn}
\end{figure}

\begin{table}[tb]
\centering
\begin{small}
\begin{tabular}{p{0.18cm}lrrrr}
  \toprule
& dataset & SPN & XSPN & MEVM &  MADE \\ 
  \midrule
\multirow{ 7}{*}{\rotatebox{90}{synthetic}} &  threshold & -68.035 & \textbf{-67.318} & -67.408 & -67.628 \\ 
 & exact & -69.327 & \textbf{-67.706} & -67.906 & -69.474 \\ 
 & parity & -69.327 & \textbf{-68.623} & -69.368 & -69.505 \\ 
 & counting & -69.323 & \textbf{-69.319} & -67.936 & -69.468 \\ 
 & MEVM-s & -12.091 & \textbf{-11.889} & -12.524 & -11.885 \\ 
 & MEVM-l & 87.795 & \textbf{-84.858} & -94.309 & -81.835 \\ 
 & conference & -5.936 & \textbf{-5.907} & -5.91 & -5.541 \\ 
  \midrule
\multirow{ 4}{*}{\rotatebox{90}{real}} & senate & -20.778 & \textbf{-19.663} & -21.107 & -20.972 \\ 
 & house & -86.517 & \textbf{-80.032} & -83.472 & -83.277 \\ 
 & bundestag & -228.459 & \textbf{-103.763} & -106.42 & -161.727 \\
 & exam & -6.952 & -6.72 & -7.066 & -12.308 \\ 
  \midrule
\multirow{ 20}{*}{\rotatebox{90}{benchmark}} & nltcs & -6.059 & -6.06 & -6.04$^*$ & -6.04 \\ 
 & plants & -13.035 & \textbf{-13.009} & -14.86$^*$ & -12.32 \\ 
 & webkb & -156.444 & -156.65 & -157.21$^*$ & -149.59 \\ 
 & msnbc & -6.043 & -6.043 & -6.23$^*$ & -6.06 \\ 
 & kdd & -2.151 & -2.15 & -2.13$^*$ & -2.07 \\ 
 & audio & -40.823 & -40.82 & -40.63$^*$ & -38.95 \\ 
 & jester & -53.811 & -53.806 & -53.22$^*$ & -52.23 \\ 
 & netflix &  -58.213 & -58.208 & -57.84$^*$ & -55.16 \\ 
 & msweb & -9.882 & -9.884 & -9.96$^*$ & -9.59 \\ 
 & book & -35.079 & \textbf{-34.937} & -34.63$^*$ & -33.95 \\ 
 & accidents & -28.957 & \textbf{-28.937} & -38.258 & -26.42 \\ 
 & dna & -81.53 & -81.567 & -98.34 & -82.77 \\ 
 & kosarek & -10.77 & \textbf{-10.741} & -10.997 & - \\ 
 & pumsb & -23.457 & \textbf{-23.375} & -36.396 & -22.3 \\ 
 & retail & -11.019 & \textbf{-10.967} & -10.897 & -10.81 \\ 
 & movie & -52.373 & \textbf{-51.702} & -52.015 & -48.7 \\ 
 & bbc & -252.994 & -251.949 & -252.023 & -242.4 \\ 
 & ad & -18.996 & \textbf{-15.747} & -32.359 & -13.65 \\ 
 & 20ng & -153.503 & -153.545 & -152.69$^*$ & -153.18 \\ 
 & reut.-52 & -84.777 & \textbf{-84.651} & -86.98$^*$ & -82.8 \\
   \bottomrule
\end{tabular}
\end{small}
\caption{Test log likelihoods for all models and datasets. For the SPN and XSPN, results that are significantly better than the other (X)SPN model are printed in bold (paired t-test, $p<0.05$). $^*$ Results taken from  \protect\cite{niepert2014exchangeable}.  }
\label{tbl:results}
\end{table}

Table \ref{tbl:results} shows the experimental results.
For the 7 synthetic datasets, the XSPNs outperform all other baseline models (except for being on-par with MADE for some datasets). For these datasets, XSPNs can often learn the \emph{true} distribution. 
As an example, consider Figure \ref{fig:conference-xspn}, which shows the XSPN learned for the conference dataset. This XSPN represents the true dependency structure in the dataset. Note that LearnXSPN was able to discover this structure completely bottom-up from the data
without any prior knowledge about the domain, as usually required for top-down relational learning approaches.
For these datasets, XSPNs also need substantially fewer parameters: For example, the XSPN shown in Figure \ref{fig:conference-xspn} has 52 parameters, while the corresponding SPN (with identical hyperparameters) has 363 parameters.

For the real-world datasets (exam, senate, house, bundestag), XSPNs also generally outperform the baseline models.
Notably, the difference between performance of XSPN and the other models is specifically pronounced for bundestag data. We suspect that the inductive bias of XSPNs towards exchangeability is able to prevent overfitting for this dataset, which contains only few samples in relation to the number of variables.
Overall, XSPNs achieve state-of-the-art results for datasets consisting of multiple, interchangeable parts.

The 20 benchmark datasets do not explicitly encode exchangeability, thus we did not expect a large benefit of XSPNs compared to conventional SPNs with Chow-Liu tree leaves. 
Interestingly, XSPNs achieve significantly (although only slightly) higher test log likelihoods than SPNs for 8 of the 20 datasets. 
Thus, XSPNs can also be used in such general domains without loss of performance, compared to SPNs.

\subsection{Classification}

\begin{table}[tb]
\centering
\begin{small}
\begin{tabular}{lrrrrr}
  \toprule
dataset & SPN & XSPN & MEVM & SVM & XGBoost  \\ 
  \midrule
counting & 0.758 & 1.000 & 0.793 & 0.796 & 0.793  \\ 
  exact & 0.975 & 1.000 & 0.977 & 0.979 & 0.978 \\ 
  parity & 0.493 & 1.000 & 0.507 & 0.605 & 0.502  \\ 
  threshold & 0.983 & 1.000 & 0.987 & 0.982 & 0.983 \\ 
  conference & 0.854 & 0.854 & 0.851 & 0.850 & 0.852  \\ 
  \midrule
  exam & 0.986 & 0.986 & 0.971 & 0.997 & 0.971  \\
  senate & 0.983 & 0.992 & 0.958 & 0.974 & 0.967 \\ 
  house & 0.928 & 0.967 & 0.908 & 0.986 & 0.954  \\ 
  bundestag & 0.951 & 0.951 & 0.610 & 0.941 & 0.951  \\ 
   \bottomrule
\end{tabular}
\end{small}
\caption{Test accuracies for the classification experiments. }
\label{tbl:classification-results}
\end{table}

Additionally, we evaluated the classification performance of XSPNs. This was done by training a separate SPN or XSPN for each class $y$ to represent $P(\mathbf{x} \given y)$, and then computing the class posterior via $P(y \given \mathbf{x}) \propto P(\mathbf{x} \given y)\, P(y)$. 
As baselines, we used Support Vector Machines (SVMs) and gradient boosted trees (XGBoost). 
The classification tasks are based on the datasets described in Section \ref{subsec:probability-estimation}, where one of the RVs is selected as classification target. For example, for the voting data, the task is to predict the votes of one of the representatives, given the votes of all other representatives. 
More details on the classification tasks can be found in the supplementary material.

The test accuracies are shown in Table \ref{tbl:classification-results}. For the first four synthetic datasets, the XSPN can learn the true distribution of the data (conditional on the class, all RVs are fully exchangeable), and thus classifies all test samples correctly.
The conventional SPN and the baseline classifiers, on the other hand, cannot learn the underlying structure of the data appropriately; their accuracies lie in the range of the prior probability of the majority class. 
For the conference and exam data, SPN and XSPN achieve the same performance, being competitive with the baseline classifiers. 

In the vote prediction tasks, the XSPN achieves higher accuracies than the SPN for two datasets (senate, house) and the same accuracy for the bundestag dataset. Furthermore, even though (X)SPNs are not primarily designed for classification tasks, the accuracy of the XSPN model is competitive with the baseline classifiers.

\section{Discussion and Conclusion}

We proposed an extension of SPNs which allows to efficiently handle distributions over partially exchangeable RVs, as well as a structure learning algorithm for these models.
The learning algorithm constructs a deep model with efficient representations of exchangeable sequences, without needing any prior knowledge about exchangeability relations among the RVs.
XSPNs achieve state-of-the-art density estimation performance for datasets containing repeated, interchangeable parts, and are competitive with SPNs in general domains.

Our learning algorithm is based on the structure learning algorithm of \cite{vergari2015simplifying}. Future research will focus on integrating means to detect and represent exchangeability in other SPN learning algorithms, e.g.\ Bayesian learning  \cite{trapp2019bayesian}. 
 A further interesting direction for future work is to generalize our approach to other statistics (e.g.\ related to exchangeable decompositions \cite{niepert2014tractability}) and continuous and hybrid domains.

\subsubsection*{Acknowledgements}
We would like to thank Lea Cohausz for providing the \emph{exam} dataset. This research was partially funded by the German Federal Ministry for Economic Affairs and Climate Action (BMWK).

\bibliographystyle{named}
\bibliography{ijcai22}

\appendix

\section{Details of Experimental Evaluation}
\label{sec:appendix-data}

 In the following, we describe the evaluation datasets in more detail.

\paragraph*{Probability Estimation Datasets}
The first four synthetic datasets were created by following Niepert and Domingos \cite{niepert2014exchangeable}: Each dataset was created by sampling uniformly from $\{0,1\}^{100}$ and rejecting all samples that did not satisfy a dataset-specific constraint. Let $n(\mathbf{x})$ be the number of ones in a sample. 
For the exact data, samples with $n(\mathbf{x}) \text{mod} 5 = 0$ were selected; for the threshold data samples with $n(\mathbf{x}) < 45$ were selected; for the parity data samples with $\mathbf{x} \text{ mod } 2 = 0$ were selected; and for the counting data samples with $n(\mathbf{x}) \text{ mod } 5 = 3$ were selected. 

The MEVM-small data was sampled from a MEVM with a single mixture component and 4 independent factors, each having 5 fully exchangeable RVs. The MEVM-large data was generated similarly, but sampled from a MEVM with 10 mixture components, 10 factors per component and 15 RVs per factor. The assignment of RVs to factors, as well as the parameters of the factors were chosen randomly for each mixture component.

The conference dataset was created by sampling from a Markov Logic Network with the following formulas:
\begin{verbatim}
1 HotTopic(c) => Attends(p,c)
0.2 Attends(px) ^ Attends(py)
1 Attends(p,c) => Overflow(c)
\end{verbatim}
It models researchers attending a conference. The probability of each person attending the conference depends on whether the conference is considered a \emph{hot topic} and on the number of other people attending. The more people attend, the more likely there is an overflow of the conference. 

The \emph{exams} dataset consists of information on students' exam attendance. Specifically, the dataset contains data of all 487 business students that started their studies in fall term 2019 at the University of Mannheim. 
 58 courses (specifically, the corresponding exams) were attended by  at least 10 of these students between fall 2019 and fall 2021.
Each binary RV represents one of these 58 exams, each example represents a student. Each value of a RV denotes whether the student participated in that exam. 

Finally, we used three real-world datasets describing voting behavior of representatives in parliaments.
The senate dataset contains all 720 roll call votes in the Senate of the 116th United States Congress, taken from  \cite{lewis2021voteview}.
We only kept votes of the 98 senators that participated in the majority of the votes. Each binary RV represents the votes of a senator (yea or nay/abstention), and each example represents a vote. 
Similarly, the House dataset (also taken from \cite{lewis2021voteview}) contains 952 roll calls of 418 representatives in the House of Representatives of the 116th United States Congress. The Bundestag data, taken from \cite{bergmann2018btvote}, contains 252 roll calls of 542 representatives in the 17th the German federal parliament (the \emph{Bundestag}).

Exchangeability arises naturally in these datasets: Students need to take a fixed number of courses from a subset, but can choose the specific courses freely. Similarly, voting behavior of representatives of the same party or fraction can be highly correlated, but each individual's (marginal) probability of a positive vote might be identical. Thus, we assume that XSPNs contain a useful inductive bias for such datasets.
For each dataset, we predefined the size of the training, validation and test sets, and then assigned the samples to these sets randomly. 
The properties of all datasets are shown in Table \ref{tbl:data-properties}.

\begin{table}[tb]
\centering
\begin{tabular}{llrrrr}
  \toprule
& dataset & $|V|$ & train & valid. & test \\ 
  \midrule
\multirow{ 7}{*}{\rotatebox{90}{synthetic}}  & threshold & 100 & 10000 & 5000 & 5000 \\ 
 & exact & 100 & 10000 & 5000 & 5000 \\ 
 & parity & 100 & 10000 & 5000 & 5000 \\ 
 & counting & 100 & 10000 & 5000 & 5000 \\ 
 & MEVM-small &  20 & 5000 & 2500 & 2500 \\ 
 & MEVM-large & 150 & 10000 & 5000 & 5000 \\ 
 & conference &  12 & 50000 & 50000 & 50000 \\ 
  \midrule
\multirow{ 4}{*}{\rotatebox{90}{real}}  & exam & 58 & 350 & 67 & 70 \\
 & senate & 98 & 500 & 100 & 120 \\
 & house & 418 & 700 & 100 & 152 \\
 & bundestag & 542 & 170 & 41 & 41\\ 
   \bottomrule
\end{tabular}
\caption{Properties of the new evaluation datasets.}
\label{tbl:data-properties}
\end{table}

\paragraph*{Classification Datasets and Experiments}
We experimented with versions of the threshold, exact, parity and counting data, which were created by sampling uniformly from $\{0,1\}^{100}$, and assigning each sample to class 1 when a dataset-specific constraint (as described above) is met, and to class 0 otherwise.
For the conference dataset, we used the \emph{overflow} variable as classification target.
For the exam data, participation in an (arbitrarily chosen) exam was used as classification target, and all other exam participation data was used as input features. 
 For each of the voting datasets (senate, house, bundestag), the votes of one (arbitrarily chosen) representative were used as classification target, and the votes of all other representatives were used as input features.
Each dataset was split into disjoint training, validation and test datasets.

As baselines, we used Support Vector Machines (SVMs) and gradient boosted trees (XGBoost). 
We used the SVM implementation of the R package \texttt{e1071} (using the default values of all parameters). Gradient boosted trees were taken from the R package \texttt{xgboost}. We  set the maximum number of boosting iterations to 2 and used the default values for all other parameters. 
MEVMs were trained by using the class variable to indicate the mixture component, i.e., for binary classification tasks, the MEVM consists of two mixture components.

\section{Ablation Study}

\begin{table*}[tb]
\centering
\begin{tabular}{llllll|ll}
  \toprule
& dataset & SPN & SPN-CLT & XSPN-T & XSPN-TF & MADE & MEVM \\ 
  \midrule 
\multirow{ 7}{*}{\rotatebox{90}{synthetic}}  &  threshold & -68.035 & -68.245 & \textbf{-67.318} & \textbf{-67.318} & -67.628 & -67.41 \\ 
 &  exact & -69.327 & -69.542 & \textbf{-67.706} & \textbf{-67.706} & -69.474 & -67.91 \\ 
  & parity & -69.327 & -69.327 & \textbf{-68.623} & \textbf{-68.623} & -69.505 & -69.37 \\ 
  & counting & -69.323 & -69.323 & \textbf{-69.319} & \textbf{-69.319} & -69.468 & -67.94 \\ 
  & MEVM-s & -12.091 & -12.091 & \textbf{-11.889} & \textbf{-11.889} & -11.885 & -12.52 \\ 
  & MEVM-l & -87.795 & -88.135 & \textbf{-84.858} & \textbf{-84.858} & -81.835 & -94.31 \\
  & conference & -5.936 & -5.936 & \textbf{-5.907} & \textbf{-5.907} & -5.541 & -5.91 \\ 
  \midrule
 \multirow{4}{*}{\rotatebox{90}{real}}  & senate & -20.778 & -20.832 & -20.804 & \textbf{-19.663} & -20.972 & -21.11 \\ 
  & house & -86.517 & - & -86.464 & \textbf{-80.032} & -83.277 & -83.47 \\ 
  & bundestag & -228.459 & - & -226.898 & \textbf{-103.763} & -161.727 & -106.42 \\ 
  & exam & -6.952 & - & -6.853 & -6.72 & -12.308 & -7.07 \\ 
  \midrule
 \multirow{ 20}{*}{\rotatebox{90}{benchmark}}  & nltcs & -6.059 & -6.061 & -6.06 & -6.06 & -6.04 & -6.04 \\ 
  & plants & -13.035 & -13.099 & -13.034 & \textbf{-13.009} & -12.32 & -14.86 \\ 
  & webkb & -156.444 & - & -156.658 & -156.65 & -149.59 & -157.21 \\ 
  & msnbc & -6.043 & -6.043 & -6.043 & -6.043 & -6.06 & -6.23 \\ 
  & kdd & -2.151 & -2.153 & -2.151 & -2.15 & -2.07 & -2.13 \\ 
  & audio & -40.823 & -40.856 & -40.824 & -40.82 & -38.95 & -40.63 \\ 
  & jester & -53.811 & -53.908 & -53.806 & -53.866 & -52.23 & -53.22 \\ 
  & netflix & -58.213 & -58.281 & -58.211 & -58.208 & -55.16 & -57.84 \\ 
  & msweb & -9.882 & - & -9.888 & -9.884 & -9.59 & -9.96 \\ 
  & book & -35.079 & - & -35.054 & \textbf{-34.937} & -33.95 & -34.63 \\ 
  & accidents & -28.957 & - & -28.952 & \textbf{-28.937} & -26.42 & -38.26 \\ 
  & dna & -81.53 & -88.891 & -81.533 & -81.567 & -82.77 & -98.34 \\ 
  & kosarek & -10.779 & -10.77 & \textbf{-10.741} & \textbf{-10.741} & - & -11.00 \\ 
  & pumsb & -23.457 & -23.81 & -23.446 & \textbf{-23.375} & -22.3 & -36.40 \\ 
  & retail & -11.019 & -11.022 & \textbf{-10.967} & \textbf{-10.967} & -10.81 & -10.90 \\ 
  & movie & -52.373 & - & -52.367 & \textbf{-51.702} & -48.7 & -52.02 \\ 
  & bbc & -252.994 & - & -253 & -251.949 & -242.4 & -252.02 \\ 
  & ad & -18.996 & - & -19.201 & \textbf{-15.747} & -13.65 & -32.36 \\ 
  & 20ng & -153.503 & - & \textbf{-153.486} & -153.545 & -153.18 & -152.69 \\ 
  & reut.-52 & -84.777 & - & \textbf{-84.687} & \textbf{-84.651} & -82.8 & -86.98 \\  
   \bottomrule
\end{tabular}
\caption{Mean test log likelihoods (LLs) for evaluation datasets. LLs of XSPN-T and XSPN-TF that are significantly better than LLs of SPN and SPN-CLT are printed in bold (paired t-test, $p<0.05$).}
\label{table:results-ll}
\end{table*}

\paragraph*{Experiments}
To investigate the causes for performance differences between XSPNs and conventional SPNs in more detail, we performed an ablation study.
 XSPNs (as described in the main paper) introduce exchangeable leaf distributions in two cases: (i) When the $\chi^2$ test indicates that the RVs are exchangeable, and (ii) as a fallback case when only few instances are left, similar to Chow-Liu tree leaves used in \cite{vergari2015simplifying}. Goal of this experiment was to investigate the effect of these two cases on model performance. Specifically, we evaluated the following models:

\begin{itemize}
\item SPN: A conventional SPN model trained by LearnSPN with univariate leaves. As a fallback case when the number of instances decreases below a threshold $m$, a fully factorized distribution is used.
\item SPN-CLT: A conventional SPN with Chow-Liu trees as fallback case when the number of instances falls below a threshold $m$.
\item XPN-T: An XSPN trained by LearnXSPN. Multivariate, exchangeable leaf distributions (using statistic $T^\#$) are only created when indicated by the $\chi^2$ test. As the fallback case, a fully factorized distribution is used.
\item XSPN-TF: As XSPN-T, but in the fallback case, an exchangeable leaf distribution is used.
\item MEVM: Mixtures of Exchangeable Variable Models (MEVMs) were trained by Expectation Maximization, using a fixed number of 20 mixture components, similar to \cite{niepert2014exchangeable}.  We only evaluated MEVMs for the datasets not used in \cite{niepert2014exchangeable}, and took log likelihood values for the other datasets from their results. 
\item MADE: Similar to \cite{peharz2020random}, we report results for \emph{Masked Autoencoders for Distribution Estimation} (MADE) models \cite{germain2015made} as a strong baseline. We used a single hidden layer with 500 nodes and 8 variable orderings. Note that MADEs do not provide tractable marginal inference.
\end{itemize}

In all SPN models, we used the g-test as independence test and clustered instances via EM for Gaussian Mixture Models.
As proposed by \cite{vergari2015simplifying}, we limited the number of child nodes of each sum and product node to 2 in all models. 
For all SPN models, we performed an exhaustive grid search of hyperparameters.
Specifically, we varied the g-test threshold values $\rho\in\{5,15\}$, the minimum number of instances $m \in \{20,200\}$, and the significance level of the $\chi^2$-test for exchangeability $p \in \{0.05,0.1,0.2,0.4\}$. For all models, we used Laplace smoothing with $\alpha = 0.1$. 

For each algorithm and dataset, we selected the parameter configuration with the highest mean log likelihood on validation data, and reported mean test log likelihood and number of parameters for that configuration. We used paired t-tests to assess significant differences between log likelihoods of XSPNs and conventional SPNs, and set statistical significance at $p<0.05$. 
 
\begin{figure}[t]
\centering
\includegraphics[scale=0.55]{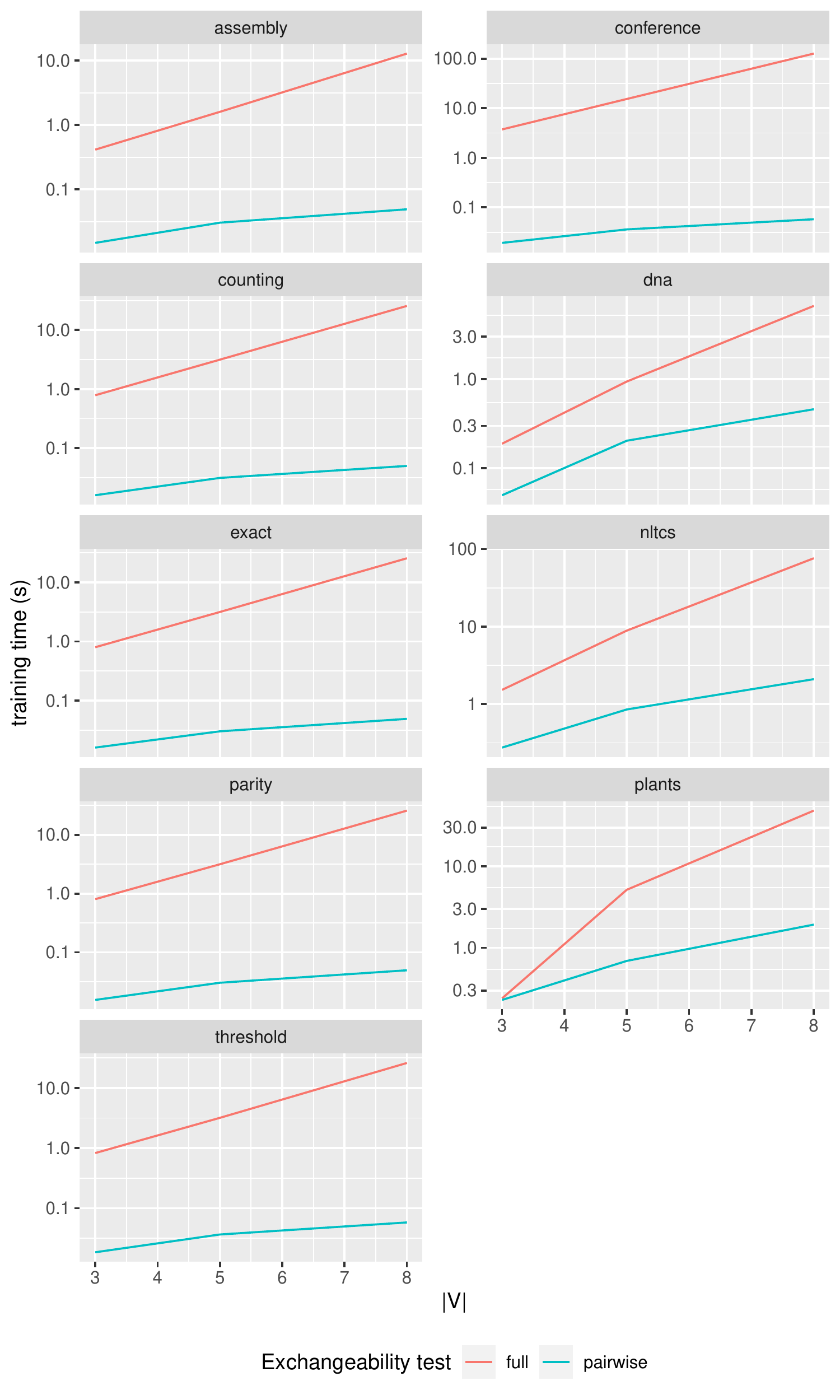}
\caption{Training time of XSPNs using the full exchangeability test, and using the pairwise approximation, w.r.t.\ the number of variables of the dataset that were used.}
\label{fig:runtime}
\end{figure}

\paragraph*{Results}
Table \ref{table:results-ll} shows the test log likelihood of all models.
By comparing SPN and XSPN-T, we can assess the effect of the leaves induced by the $\chi^2$-test on model performance (as the learning algorithm is otherwise identical). 
For the synthetic datasets, XSPN-T achieves substantially better log likelihoods than standard SPNs. 
The difference is specifically pronounced for data with full exchangeability (threshold, exact, parity, counting).
This is the case because XSPN-T can recover the true distribution of these datasets.
For the real-world datasets  as well as the 20 benchmark datasets, XSPN-T is not significantly better than SPN.

However, when additionally allowing exchangeable leaves as fallback cases (XSPN-TF), the log likelihood can be further increased for the real-world datasets. 
XSPN-TF consistently outperforms SPN and SPN-CLT (which is a fairer comparison, as they both make use of multivariate leaves) for the synthetic and real-world datasets.
Notably, the difference is specifically pronounced for bundestag data. We suspect that the inductive bias of XSPNs towards exchangeability prevents overfitting for this dataset, which contains only few samples in relation to the number of variables. 
Interestingly, XSPN-TF even outperforms the other baseline methods (MEVM, MADE) on the real-world datasets.

\section{Effect of Pairwise Exchangeability Tests}

Finally, we investigated whether the pairwise approximation of the $\chi^2$ test for exchangeability is necessary, or whether the full test could also be used. 
Specifically, this was done by evaluating two versions of LearnXSPN, either using the full $\chi^2$ test, or the pairwise approximation, on a subset of the benchmark datasets. For each of the benchmark datasets, we only used the first 3, 5 or 8 variables (i.e.\ columns), and recorded test log likelihood and runtime. 

Runtimes of the two versions of LearnXSPN are shown in Figure \ref{fig:runtime}. 
Quantitatively, the runtime behaves similar for all datasets: When using the pairwise approximation, overall runtime increases only moderately (complexity of the pairwise test is quadratic in the number of variables). When using the full test, runtime increases dramatically, so that datasets with more than 8 variables are infeasible (complexity of the full test is exponential in the the number of variables). 

Interestingly, the XSPNs resulting from the different versions of LearnXSPN are identical in all cases. That is, whenever the full test identifies exchangeability, this exchangeability is also identified by the pairwise test, and there was no case where the pairwise test identified additional exchangeability that was not identified by the full test. 
Thus, in practice, the pairwise approximation can identify he same instances of exchangeability as the full test (for the datasets we considered), while being computationally much more efficient.

\end{document}